\documentclass[a4paper]{easychair}

\usepackage{doc}
\usepackage{proof}
\usepackage{amssymb}
\usepackage{amsmath}
\usepackage{stmaryrd}
\usepackage{extarrows}
\usepackage{tikz}
\usetikzlibrary{matrix}
\usepackage{version}
\excludeversion{draft}

\newcommand{\tensor}{\otimes}
\newcommand{\bs}{\backslash}
\newcommand{\s}{\slash}
\newcommand{\fdia}{\Diamond}
\newcommand{\gbox}{\Box}
\newcommand{\arr}[3]{#1:#2\longrightarrow #3}

\renewcommand{\Diamond}{\lozenge}

\newcommand{\pijl}{\rightarrow}
\newcommand{\comp}{\circ}
\newcommand{\resright}{\rhd}
\newcommand{\resleft}{\lhd}
\newcommand{\resdia}{\triangledown}
\newcommand{\resdiainv}{\resdia^{-1}}
\newcommand{\resrightinv}{\resright^{-1}}
\newcommand{\resleftinv}{\resleft^{-1}}

\newcommand{\xlefta}{\alpha_{\diamond}^{l}}
\newcommand{\xleftc}{\sigma_{\diamond}^{l}}
\newcommand{\xrighta}{\alpha_{\diamond}^{r}}
\newcommand{\xrightc}{\sigma_{\diamond}^{r}}

\newcommand{\Xlefta}{\widehat{\alpha}_{\diamond}^{l}}
\newcommand{\Xleftc}{\widehat{\sigma}_{\diamond}^{l}}
\newcommand{\Xrighta}{\widehat{\alpha}_{\diamond}^{r}}
\newcommand{\Xrightc}{\widehat{\sigma}_{\diamond}^{r}}

\newcommand{\F}[1]{\lceil #1 \rceil}
\newcommand{\Arrow}[1]{\xlongrightarrow{\displaystyle #1}}

\newcommand{\reals}{\mathbb{R}}
\newcommand{\arrover}[3]{#2\xlongrightarrow{#1} #3}

\newcommand{\G}[1]{\lfloor #1 \rfloor}

\newtheorem{theorem}{Theorem}
\theoremstyle{definition}
\newtheorem{definition}{Definition}

\title{Lexical and Derivational Meaning in Vector-Based Models of Relativisation}

\author{
Michael Moortgat\inst{1}
\and
    Gijs Wijnholds\inst{2}
}

\institute{
  Utrecht University,
  The Netherlands\\
  \email{m.j.moortgat@uu.nl}
\and
   Queen Mary University of London,
   United Kingdom\\
   \email{g.j.wijnholds@qmul.ac.uk}\\
 }

\authorrunning{Moortgat and Wijnholds}

\titlerunning{Vector-Based Models of Relativisation}

\begin{document}

\maketitle

\begin{abstract}
Sadrzadeh et al (2013) present a compositional distributional analysis of relative clauses in English in 
terms of the Frobenius algebraic structure of finite dimensional vector spaces. 
The analysis relies on distinct type assignments and lexical recipes for subject vs object relativisation.
The situation for Dutch is different: because of the verb final nature of Dutch, relative clauses are
ambiguous between a subject vs object relativisation reading. Using an extended version of Lambek calculus,
we present a compositional distributional framework that accounts for this derivational ambiguity, and that
allows us to give a single meaning recipe for the relative pronoun reconciling the Frobenius semantics
with the demands of Dutch derivational syntax.
\end{abstract}



%
%

\section{Introduction}
\label{sect:introduction}

Compositionality, as a structure-preserving mapping from a syntactic source to a target interpretation,
is a fundamental design principle both for the set-theoretic models of formal semantics and for syntax-sensitive
vector-based accounts of natural language meaning, see \cite{Frege} for discussion. For typelogical grammar
formalisms, to obtain a compositional interpretation, we have to specify how the Syn-Sem homomorphism acts on \emph{types}
(basic and complex) and on \emph{proofs} (derivations, again basic (axioms) or compound, obtained by inference steps).
There is a tension here between lexical and derivational aspects of meaning: the derivational
aspects relate to the composition operations associated with the inference steps that put together
phrases out of more elementary parts; the atoms for this composition process are the meanings of
the lexical constants associated with the axioms of a derivation. 

Relative clause structures form a suitable testbed to study the interaction between these two aspects of meaning,
and they have been well-studied in the formal and in the distributional settings. Informally, a restrictive relative clause
(`books that Alice read') has an intersective interpretion. In the formal semantics account, this interpretation is obtained by
modeling both the head noun (`books') and the relative clause body (`Alice read \textvisiblespace') as (characteristic functions of)
sets (type $e\rightarrow t$); the relative pronoun can then be interpreted as the intersection operation.
In distributional accounts such as \cite{coecke2013lambek}, full noun phrases and simple common nouns are interpreted
in the same semantic space, say $\textsf{N}$, distinct from the sentence space $\textsf{S}$.
In this setting, element-wise multiplication, which preserves non-null context features, is a natural candidate
for an intersective interpretation; in the case at hand this means element-wise multiplication of a vector in $\textsf{N}$ 
interpreting the head noun, with a vector interpretation obtained from the relative clause body. To achieve this effect,
\cite{sadrzadeh2013frobenius} rely on the Frobenius algebraic structure of \textbf{FVect}, which provides operations for (un)copying,
insertion and deletion of vector information. A key feature of their account is that it relies on \emph{structure-specific} solutions
of the lexical equation: subject and object relative clauses are obtained from distinct type assignments
to the relative pronoun (Lambek types $(n\bs n)/(np\bs s)$ vs $(n\bs n)/(s/np)$), associated with distinct
instructions for meaning assembly.

For a language like Dutch, such an account is problematic. Dutch subordinate clause order has the SOV pattern Subj--Obj--TV,
i.e.~a transitive verb is typed as $np\bs(np\bs s)$, selecting its arguments uniformly to the left.
As a result, example (\ref{nlvse})(a) is ambiguous between a subject vs object relativisation interpretation:
it can be translated as either (b) or (c). The challenge here is twofold: at the syntactic level, we have to provide a 
\emph{single} type assignment to the relative pronoun that can withdraw either a subject or an object hypothesis from
the relative clause body; at the semantic level, we need a \emph{uniform} meaning recipe for the relative pronoun
that will properly interact with the derivational semantics. 
\begin{equation}\label{nlvse}
\begin{tabular}{cll}
$a$ & mannen$_{n}$ die$_{?}$ vrouwen$_{np}$ haten$_{np\bs(np\bs s)}$ & (ambiguous)\\
$b$ & men who hate women & (subject rel)\\
$c$ & men who(m) women hate & (object rel)\\
\end{tabular}
\end{equation}
The paper is structured as follows. In \S \ref{sect:syntax}, we present an extended version of Lambek calculus, and show how it accounts
for the derivational ambiguity of Dutch relative clauses. In \S \ref{subsect:derivational}, we define the interpretation homomorphism
that associates syntactic derivations with composition operations in a vector-based semantic model. The derivational semantics
thus obtained is formulated at the type level, i.e.~it abstracts from the contribution of individual lexical items.
In \S \ref{subsect:lexical}, we bring in the lexical semantics, and show how the Dutch relative pronoun can be given a uniform
interpretation that properly interacts with the derivational semantics. The discussion in \S \ref{sect:discussion} compares
the distributional and formal semantics accounts of relativisation.

\begin{figure}[t]
\[\infer[]{\arr{1_A}{A}{A}}{}
\qquad
\infer[]{\arr{g\circ f}{A}{C}}{\arr{f}{A}{B} & \arr{g}{B}{C}}\]

\[\infer[]{\arr{\resdia f}{A}{\gbox B}}{\arr{f}{\fdia A}{B}}
\qquad
\infer[]{\arr{\resright f}{A}{C/B}}{\arr{f}{A\tensor B}{C}}
\qquad
\infer[]{\arr{\resleft f}{B}{A\bs C}}{\arr{f}{A\tensor B}{C}}\]

\[\infer[]{\arr{\resdiainv g}{\fdia A}{B}}{\arr{g}{A}{\gbox B}}
\qquad
\infer[]{\arr{\resrightinv g}{A\tensor B}{C}}{\arr{g}{A}{C/B}}
\qquad
\infer[]{\arr{\resleftinv g}{A\tensor B}{C}}{\arr{g}{B}{A\bs C}}\]

\[\arr{\xlefta}{\fdia A\tensor(B\tensor C)}{(\fdia A\tensor B)\tensor C}
\qquad
\arr{\xrighta}{(A\tensor B)\tensor\fdia C}{A\tensor(B\tensor\fdia C)}\]
\[\arr{\xleftc}{\fdia A\tensor(B\tensor C)}{B\tensor(\fdia A\tensor C)}
\qquad
\arr{\xrightc}{(A\tensor B)\tensor\fdia C}{(A\tensor\fdia C)\tensor B}\]

\caption{\textbf{NL}$_{\diamond}$. Residuation rules; extraction postulates.}
\label{nldia}
\end{figure}

\section{Syntax}
\label{sect:syntax}

Our syntactic engine is \textbf{NL}$_{\diamond}$ \cite{moortgat1996multimodal}: the extension of Lambek's \cite{lam61}
Syntactic Calculus with an adjoint pair of control modalities $\fdia, \gbox$. 
The modalities play a role similar to that of the exponentials of linear logic: 
they allow one to introduce controlled, rather than global, forms of reordering and restructuring.
In this paper, we consider the controlled associativity and commutativity postulates of \cite{moortgatconstants}.
One pair, $\xlefta,\xleftc$, allows a $\fdia$-marked formula to reposition itself on
left branches of a constituent tree; we use it to model the SOV extraction patterns in Dutch.
A symmetric pair $\xrighta,\xrightc$ would capture the non-local extraction dependencies in
an SVO language such as English. Lambek \cite{lam88} has shown how deductions in a syntactic
calculus can be viewed as arrows in a category. Figure \ref{nldia} presents \textbf{NL}$_{\diamond}$ in this format.

For parsing, we want a proof search procedure that doesn't rely on cut. 
Consider the rules in Figure \ref{monox}, expressing the monotonicity properties of the type-forming
operations, and recasting the postulates in rule form. It is routine to show that these are \emph{derived} rules
of inference of $\textbf{NL}_{\diamond}$. In \cite{mmmoot13} it is shown that by adding them to the residuation rules of Figure \ref{nldia}, one obtains
a system equivalent to a display sequent calculus enjoying cut-elimination. By further restricting to
\emph{focused} derivations, proof search is free of spurious ambiguity. 

\begin{figure}[b]
\[\infer[]{\arr{\Diamond f}{\fdia A}{\fdia B}}{\arr{f}{A}{B}}
\qquad
\infer[]{\arr{\Box f}{\gbox A}{\gbox B}}{\arr{f}{A}{B}}\]

\[\infer[]{\arr{f\tensor g}{A\tensor C}{B\tensor D}}{\arr{f}{A}{B} & \arr{g}{C}{D}}
\qquad
\infer[]{\arr{f/g}{A/D}{B/C}}{\arr{f}{A}{B} & \arr{g}{C}{D}}
\qquad
\infer[]{\arr{f\bs g}{B\bs C}{A\bs D}}{\arr{f}{A}{B} & \arr{g}{C}{D}}
\]

\[\infer[]{\arr{\Xlefta f}{\fdia A\tensor(B\tensor C)}{D}}{\arr{f}{(\fdia A\tensor B)\tensor C}{D}}
\qquad
\infer[]{\arr{\Xleftc f}{\fdia A\tensor(B\tensor C)}{D}}{\arr{f}{B\tensor (\fdia A\tensor C)}{D}}
\]
\caption{\textbf{NL}$_{\diamond}$. Monotonicity; leftward extraction (rule version).}
\label{monox}
\end{figure}

We are ready to return to our example (\ref{nlvse})(a). A type assignment $(n\bs n)/(\fdia\gbox np\bs s)$ to the
relative pronoun `die' accounts
for the derivational ambiguity of the phrase. The derivations agree on the initial steps

\begin{equation}\label{incommon}
\infer[\triangleleft^{-1}]{ n_{} \otimes (((n_{} \bs n_{})/ ( \fdia  \gbox np_{} \bs s_{})) \otimes (np_{} \otimes (np_{} \bs (np_{} \bs s_{}))))  \longrightarrow n_{}}{
 \infer[\triangleright^{-1}]{ ((n_{} \bs n_{})/ ( \fdia  \gbox np_{} \bs s_{})) \otimes (np_{} \otimes (np_{} \bs (np_{} \bs s_{})))  \longrightarrow  n_{} \bs n_{} }{
 \infer[\slash ]{{ (n_{} \bs n_{})/ ( \fdia  \gbox np_{} \bs s_{}) } \longrightarrow  (n_{} \bs n_{})/ (np_{} \otimes (np_{} \bs (np_{} \bs s_{}))) }{
 \infer[\bs ]{{ n_{} \bs n_{} } \longrightarrow  n_{} \bs n_{} }{
 \infer{{n_{}} \longrightarrow n_{}}{} &  \infer{{n_{}} \longrightarrow n_{}}{}} &  
 \infer[]{ np_{} \otimes (np_{} \bs (np_{} \bs s_{}))  \longrightarrow   \fdia  \gbox np_{} \bs s_{} }{\vdots}}}}
\end{equation}
but then diverge in how the relative clause body is derived:
\begin{equation}\label{divergence}
\addtolength{\inferLineSkip}{1pt}
\infer[\triangleleft^{-1}]{ np_{} \otimes (np_{} \bs (np_{} \bs s_{}))  \longrightarrow   \fdia  \gbox np_{} \bs s_{} }{
 \infer[\bs ]{{ np_{} \bs (np_{} \bs s_{}) } \longrightarrow  np_{} \bs ( \fdia  \gbox np_{} \bs s_{}) }{
 \infer{{np_{}} \longrightarrow np_{}}{} &  \infer[\bs ]{ \framebox{${ np_{} \bs s_{} } \longrightarrow   \fdia  \gbox np_{} \bs s_{} $} }{
 \infer[\triangledown^{-1}]{ \fdia  \gbox np_{} \longrightarrow np_{}}{ \infer[\gbox ]{{ \gbox np_{}} \longrightarrow  \gbox np_{}}{ \infer{{np_{}} \longrightarrow np_{}}{}}} &  \infer{{s_{}} \longrightarrow s_{}}{}}}}
\quad
\infer[\triangleleft]{ np_{} \otimes (np_{} \bs (np_{} \bs s_{}))  \longrightarrow   \fdia  \gbox np_{} \bs s_{} }{
\infer[\widehat{\sigma}_{\diamond}^{l}]{\fdia  \gbox np_{} \otimes ( np_{} \otimes (np_{} \bs (np_{} \bs s_{})) ) \longrightarrow s_{} }{
\infer[\triangleleft^{-1}]{ np_{} \otimes ( \fdia  \gbox np_{} \otimes (np_{} \bs (np_{} \bs s_{})) ) \longrightarrow s_{} }{
 \infer[\triangleleft^{-1}]{  \fdia  \gbox np_{} \otimes (np_{} \bs (np_{} \bs s_{}))  \longrightarrow  np_{} \bs s_{} }{
 \infer[\bs ]{ \framebox{$ { np_{} \bs (np_{} \bs s_{}) } \longrightarrow   \fdia  \gbox np_{} \bs (np_{} \bs s_{}) $} }{
 \infer[\triangledown^{-1}]{ \fdia  \gbox np_{} \longrightarrow np_{}}{ \infer[\gbox ]{{ \gbox np_{}} \longrightarrow  \gbox np_{}}{ \infer{{np_{}} \longrightarrow np_{}}{}}} &  \infer[\bs ]{{ np_{} \bs s_{} } \longrightarrow  np_{} \bs s_{} }{
 \infer{{np_{}} \longrightarrow np_{}}{} &  \infer{{s_{}} \longrightarrow s_{}}{}}}}}}}
\end{equation}
In the derivation on the left, the $\fdia\gbox np$ hypothesis is linked to the \emph{subject} argument of the verb;
in the derivation on the right to the \emph{object} argument, reached via the $\Xleftc$ reordering step.

\section{From source to target}
\label{sect:interpretation}

\subsection{Derivational semantics}
\label{subsect:derivational}

Compositional distributional models are obtained by
defining a homomorphism sending types and derivations of a syntactic source system
to their counterparts in a symmetric compact closed category (sCCC); the concrete model
for this sCCC then being finite dimensional vector spaces (\textbf{FVect}) and (multi)linear maps.
Such interpretation homomorphisms have been defined for pregroup grammars, Lambek calculus and CCG
in \cite{coecke2013lambek,maillardclark2014}. We here define the interpretation for $\textbf{NL}_{\diamond}$, 
starting out from \cite{wijnholds2014categorical}.

Recall first that a \emph{compact closed category} (CCC) is monoidal, i.e.~it has an associative $\tensor$ with unit $I$;
and for every object there is a left and a right adjoint satisfying
\[A^{l}\tensor A\Arrow{\epsilon^{l}} I \Arrow{\eta^{l}} A\tensor A^{l}
\qquad
A\tensor A^{r}\Arrow{\epsilon^{r}} I \Arrow{\eta^{r}} A^{r}\tensor A\]
In a \emph{symmetric} CCC, the tensor moreover is commutative, and we can write $A^{*}$ for the collapsed left and right adjoints.

In the concrete instance of \textbf{FVect}, the unit $I$ stands for the field $\mathbb{R}$; identity maps, composition and tensor product are defined as usual. Since bases of vector spaces are fixed in concrete models, there is only one natural way of defining a basis for a \emph{dual space}, so that $V^{*} \cong V$. In concrete models we may collapse the adjoints completely.

The $\epsilon$ map takes inner products, whereas the $\eta$ map (with $\lambda = 1$) introduces an identity tensor as follows:

\begin{center}
	\begin{tabular}{l@{\hskip 1em}c@{\hskip 4em}ccl}
			$ \epsilon_V : V \tensor V \pijl \reals$ & given by & $\sum\limits_{ij} v_{ij} (\vec{e}_i \tensor \vec{e}_j)$ & $\mapsto$ & $\sum\limits_{i} v_{ii}$ \\[1.5em]
			$ \eta_V : \reals \pijl V \tensor V$ & given by & $\lambda$ & $\mapsto$ & $\sum\limits_i \lambda (\vec{e}_i \tensor \vec{e}_i) $
	\end{tabular}
\end{center}

\paragraph*{Interpretation: types} At the type level, the interpretation function $\F{\cdot}$ assigns a vector space to the atomic types of $\textbf{NL}_{\diamond}$;
for complex types we set $\F{\fdia A}=\F{\gbox A}=\F{A}$, i.e.~the syntactic control operators are transparent for the
interpretation; the binary type-forming operators are interpreted as

\[\F{A\tensor B}=\F{A}\tensor\F{B}
\quad
\F{A/B}=\F{A}\tensor\F{B}^{*}
\quad
\F{A\bs B}=\F{A}^{*}\tensor\F{B}\]
\paragraph*{Interpretation: proofs} From the linear maps interpreting the premises of the $\textbf{NL}_{\diamond}$ inference rules, 
we want to compute the linear map interpreting the conclusion. Identity and composition are immediate:
$\F{1_A}=1_{\F{A}}$, $\F{g\circ f}=\F{g}\circ\F{f}$. 
For the residuation inferences, from the map $\arr{\F{f}}{\F{A}\otimes\F{B}}{\F{C}}$ interpreting the premise, we obtain
\[\begin{array}{c}
\F{\resright f} = \F{A}\Arrow{1_{\F{A}}\tensor\eta_{\F{B}}}\F{A}\tensor\F{B}\tensor\F{B}^{*}\Arrow{\F{f}\otimes 1_{\F{B}^{*}}}\F{C}\tensor\F{B}^{*}\\
\F{\resleft f} = \F{B}\Arrow{\eta_{\F{A}}\otimes 1_{\F{B}}}\F{A}^{*}\tensor\F{A}\tensor\F{B}\Arrow{1_{\F{A}^{*}}\otimes\F{f}}\F{A}^{*}\tensor\F{C}
\end{array}\]
For the inverses, from maps $\arr{\F{g}}{\F{A}}{\F{C/B}}$, $\arr{\F{h}}{\F{B}}{\F{A\bs C}}$ for the premises, we obtain
\[\begin{array}{c}
\F{\resrightinv g} = \F{A}\tensor\F{B}\Arrow{\F{g}\tensor 1_{\F{B}}}\F{C}\tensor\F{B}^{*}\tensor\F{B}\Arrow{1_{\F{C}}\tensor\epsilon_{\F{B}}}\F{C}\\
\F{\resleftinv h} = \F{A}\tensor\F{B}\Arrow{1_{\F{A}}\tensor\F{h}}\F{A}\tensor\F{A}^{*}\tensor\F{C}\Arrow{\epsilon_{\F{A}}\tensor 1_{\F{C}}}\F{C}
\end{array}\]
Monotonicity. The case of parallel composition is immediate: $\F{f\tensor g}=\F{f}\tensor\F{g}$. For the slash cases,
from $\arr{\F{f}}{\F{A}}{\F{B}}$ and $\arr{\F{g}}{\F{C}}{\F{D}}$, we obtain

\begin{tikzpicture}
\matrix (m) [matrix of math nodes, row sep=3em, column sep=4em]
{\F{f/g}= & \F{f\bs g}= \\[-2em]
\F{A}\tensor\F{D}^{*} &  \F{B}^{*}\tensor\F{C} \\
\F{B}\tensor\F{C}^{*}\tensor\F{C}\tensor\F{D}^{*} & \F{B}^{*}\tensor\F{A}\tensor\F{A}^{*}\tensor\F{D} \\
\F{B}\tensor\F{C}^{*}\tensor\F{D}\tensor\F{D}^{*} & \F{B}^{*}\tensor\F{B}\tensor\F{A}^{*}\tensor\F{D} \\
\F{B}\tensor\F{C}^{*} & \F{A}^{*}\tensor\F{D} \\};

\path[-stealth] (m-2-1) edge node [left] {$\F{f}\tensor\eta_{\F{C}}\tensor 1_{\F{D}^{*}}$} (m-3-1);
\path[-stealth] (m-3-1) edge node [left] {$1_{\F{B}\tensor\F{C}^{*}}\tensor\F{g}\tensor 1_{\F{D}^{*}}$} (m-4-1);
\path[-stealth] (m-4-1) edge node [left] {$1_{\F{B}\tensor\F{C}^{*}}\tensor\epsilon_{\F{D}}$} (m-5-1);

\path[-stealth] (m-2-2) edge node [right] {$1_{\F{B}^{*}}\tensor\eta_{\F{A}}\tensor \F{g}$} (m-3-2);
\path[-stealth] (m-3-2) edge node [right] {$1_{\F{B}^{*}}\tensor\F{f}\tensor 1_{\F{A}^{*}\tensor\F{D}}$} (m-4-2);
\path[-stealth] (m-4-2) edge node [right] {$\epsilon_{\F{B}}\tensor 1_{\F{A}^{*}\tensor\F{D}}$} (m-5-2);
\end{tikzpicture}

\noindent
Interpretation for the extraction structural rules is obtained via the standard associativity and symmetry maps of \textbf{FVect}:
$\F{\Xlefta f} = f \comp \alpha$ and $\F{\Xleftc f} = f \comp \alpha^{-1}\comp (\sigma \tensor 1_A) \comp \alpha$ and similarly for the rightward extraction rules.

\paragraph*{Simplifying the interpretation}\label{simpl} Whereas the syntactic derivations of $\textbf{NL}_{\diamond}$ proceed in cut-free fashion, 
the interpretation of the inference rules given above introduces detours (sequential composition of maps) that can be removed. We use
a generalised notion of Kronecker delta, together with Einstein summation notation, to concisely express the fact that the interpretation
of a derivation is fully determined by the identity maps that interpret its axiom leaves, realised as the $\epsilon$ or $\eta$
identity matrices depending on their (co)domain signature.

Recall that vectors and linear maps over the real numbers can be equivalently expressed as (multi-dimensional) arrays of numbers. The essential information one needs to keep track of are the coefficients of the tensor: for a vector $\mathbf{v}\in\reals^{n}$ we write $v_i$ (with $i$ ranging from $1$ to $n$), an $n\times m$ matrix $\textbf{A}$ is
expressed as $A_{ij}$, an $n\times m\times p$ cube $\textbf{B}$ as $B_{ijk}$, with the indices each time ranging over the dimensions. The Einstein summation
convention on indices then states that in an expression involving multiple tensors, indices occurring once give rise to a tensor product, whereas indices occurring
twice are contracted.
Without explicitly writing a tensor product $\tensor$, the tensor product of a vector $\mathbf{a}$ and a matrix $\mathbf{A}$ thus can be written as $a_i A_{jk}$; the inner product between vectors $\mathbf{a},\mathbf{b}$ is $a_i b_i$. Matrix application $\mathbf{A}\mathbf{a}$ is rendered as $A_{ij}a_j$, i.e.~the contraction happens over the second dimension of $\mathbf{A}$ and $\mathbf{a}$.
For tensors of arbitrary rank we use uppercase to refer to lists of indices: we write a tensor $\mathbf{T}$ as $T_I$. Tensor application then becomes $T_{IJ}R_J$, for some tensor $\mathbf{R}$ of lower rank. 

The identity matrix is given by the Kronecker delta (left), the identity tensor by its generalisation (right):
\begin{center}
\begin{tabular}{c@{\hskip 4em}c}
	$\delta^i_j = \begin{cases} 1 & i=j \\ 0 & \text{otherwise} \end{cases}$ & $\delta^I_J = \begin{cases} 1 & I_k = J_k \ \text{for all $k$} \\ 0 & \text{otherwise} \end{cases}$
\end{tabular}
\end{center}
The attractive property of the (generalised) Kronecker delta is that it expresses unification of indices: $\delta^i_j a_i = a_j$, which is simply a renaming of the index; the inner product can be computed by $\delta^i_j a_i b_j = a_j b_j$. Left on its own, it is simply an identity matrix/tensor.

With the Kronecker delta, the composition of matrices $\mathbf{B} \comp \mathbf{A}$ is expressible as $\delta^{j}_{k} A_{ij} B_{kl}$, which is the same as $A_{ij} B_{jl}$ (or $A_{ik}B_{kl}$). We can show that order of composition is irrelevant:
	$$ \delta^j_k A_{ij} \delta^l_m B_{kl} C_{mn} = A_{ij} B_{jl} C_{ln} = \delta^l_m \delta^j_k A_{ij} B_{kl} C_{mn}$$
The special cases of tensor product of generalised Kronecker deltas is given by concatenating the index lists:
\[\delta^I_{J} \tensor \delta^K_L = \delta^{IK}_{JL} \]
expressing the fact that $1_A \tensor 1_B = 1_{A \tensor B}$.

Since the generalised Kronecker delta is able to do renaming, take inner product, and insert an identity tensor, depending on the number of arguments placed behind it, it will represent precisely the $1_A, \epsilon_A, \eta_A$ maps discussed above. In this respect, the interpretation can be simplified and we can label the proof system (with formulas already interpreted) with these generalised Kronecker deltas. The effect of the residuation rules and the structural rules is to only change the (co)domain signature of a Kronecker delta, whereas the rules for axioms and monotonicity also act on the Kronecker delta itself:
\[\infer[1_A]{\arrover{\delta^{I}_{J}}{A_I}{A_J}}{}\]
\[\infer[\tensor]{\arrover{\delta^{IK}_{JL}}{A\tensor C}{B\tensor D}}{\arrover{\delta^I_J}{A}{B} & \arrover{\delta^K_L}{C}{D}}
\qquad
\infer[\s]{\arrover{\delta^{IK}_{JL}}{A\tensor D}{B\tensor C}}{\arrover{\delta^I_J}{A}{B} & \arrover{\delta^K_L}{C}{D}}
\qquad
\infer[\bs]{\arrover{\delta^{IK}_{JL}}{B\tensor C}{A\tensor D}}{\arrover{\delta^I_J}{A}{B} & \arrover{\delta^K_L}{C}{D}}
\]
In Appendix \ref{simplappendix} we show that this labelling is correct for the general interpretation of proofs in \S\ref{subsect:derivational}.
\subsection{Lexical semantics}
\label{subsect:lexical}

For the general interpretation of types and proofs given above, a proof  $\arr{f}{A}{B}$ is interpreted as a linear map $\F{f}$ sending an element belonging to $\F{A}$, the semantic space interpreting $A$, to an element of $\F{B}$. The map is expressed at the general level of types, and completely abstracts from \emph{lexical} semantics. 
For the computation of concrete interpretations, we have to bring in the meaning of the lexical items.  
For $A = A_1\tensor\cdots\tensor A_n$, this means applying the map $\F{f}$ to $\mathbf{w}_{1}\tensor\cdots\tensor\mathbf{w}_{n}$,
the tensor product of the word meanings making up the phrase under consideration, to obtain a meaning $M\in\F{B}$,
the semantic space interpreting the goal formula.

With the index notation introduced above, $\F{f}$ is expressed in the form of a generalised Kronecker delta, which 
is applied to the tensor product of the word meanings in index notation to produce the final meaning in $\F{B}$.
In (\ref{four}) we illustrate with the interpretation of some proofs derived from the same axiom leaves, $np\longrightarrow np$ and $s\longrightarrow s$.
Assuming $\F{np}=\textsf{N}$ and $\F{s}=\textsf{S}$, these correspond to identity maps on $\textsf{N}$ and $\textsf{S}$. We use the
convention that the formula components of the endsequent are labelled in alphabetic order; the correct indexing for the Kronecker
delta is obtained by working back to the axiom leaves.

\begin{equation}\label{four}
\begin{array}{cr@{\qquad}l}
a & \textsf{dream}^{np\bs s} \longrightarrow np\bs s & \mathbf{dream}^{\textsf{N}\otimes\textsf{S}}_{i,j}\Arrow{\delta^{k,j}_{i,l}}T\,_{k,l}^{\textsf{N}\otimes\textsf{S}}\\[2ex]

b & \textsf{poets}^{np} \otimes \textsf{dream}^{np\bs s} \longrightarrow s & \mathbf{poets}^{\textsf{N}}_{i} \otimes \mathbf{dream}^{\textsf{N}\otimes\textsf{S}}_{j,k}\Arrow{\delta^{i,k}_{j,l}}V\,_{l}^{\textsf{S}}\\[2ex]

c & \textsf{poets}^{np}\longrightarrow s/(np\bs s) & \mathbf{poets}^{\textsf{N}}_{i}\Arrow{\delta^{i,l}_{k,j}}R\,_{j,k,l}^{\textsf{S}\otimes\textsf{N}\otimes\textsf{S}}\\[2ex]

\end{array}
\end{equation}
(\ref{four})(a) expresses the linear map from $\mathbf{dream}\in\textsf{N}\otimes\textsf{S}$ to a tensor $T\in\textsf{N}\otimes\textsf{S}$.
Because we have $T = \delta^{k,j}_{i,l} \mathbf{dream}_{i,j} = \mathbf{dream}_{k,l}$, this is in fact the identity map.
(\ref{four})(b) computes a vector $V\in\textsf{S}$ with $V = \delta^{i,k}_{j,l} \mathbf{poets}_{i}\tensor\mathbf{dream}_{j,k} = \mathbf{poets}_{j}\tensor\mathbf{dream}_{j,l}$.
In (\ref{four})(c) we arrive at an interpretation $R\in\textsf{S}\tensor\textsf{N}\tensor\textsf{S}$ with $R = \delta^{i,l}_{k,j} \mathbf{poets}_{i} = \delta^{l}_{j} \mathbf{poets}_{k}$. Note that we wrote the tensor product symbol $\tensor$ explicitly.

In the case of our relative clause example (\ref{nlvse}), the derivational ambiguity of (\ref{divergence}) gives 
rise to two ways of obtaining a vector $\mathbf{v}\in\textsf{N}$. They
differ in whether $l$, the index of the $\fdia\gbox np$ hypothesis in the relative pronoun type,
contracts with index $p$ for the subject argument of the verb (\ref{subjectreading}) or with 
the direct object index $o$ (\ref{objectreading}).
\begin{equation}\label{subjectreading}
\begin{array}{rl}
 & \mathbf{mannen}_{i}\tensor\mathbf{die}_{jklm}\tensor\mathbf{vrouwen}_{n}\tensor\mathbf{haten}_{opq} \Arrow{\delta^{i,k,l,m,n}_{j,r,p,q,o}} \mathbf{v}^{subj}_{r}\in\textsf{N} \\[2ex]
\mathbf{v}^{subj}_{j} = & \mathbf{mannen}_{i}\tensor\mathbf{die}_{ijkl}\tensor\mathbf{vrouwen}_{m}\tensor\mathbf{haten}_{mkl} \qquad (\textrm{relabeled}) \\
\end{array}
\end{equation}
\begin{equation}\label{objectreading}
\begin{array}{rl}
& \mathbf{mannen}_{i}\tensor\mathbf{die}_{jklm}\tensor\mathbf{vrouwen}_{n}\tensor\mathbf{haten}_{opq} \Arrow{\delta^{i,k,l,m,n}_{j,r,o,q,p}} \mathbf{v}^{obj}_{r}\in\textsf{N} \\[2ex]
 \mathbf{v}^{obj}_{j} = & \mathbf{mannen}_{i}\tensor\mathbf{die}_{ijkl}\tensor\mathbf{vrouwen}_{m}\tensor\mathbf{haten}_{kml} \qquad (\textrm{relabeled}) \\
\end{array}
\end{equation}

\medskip\noindent
The picture in Figure \ref{matching} expresses this graphically. 

\begin{figure}[h]
\begin{center}
\begin{tikzpicture}[every node/.style={anchor=base},xscale=.9,yscale=.9]
\node (n0) at (0,0) {\textsf{N}};
\node (n1) at (1.3,0) {\textsf{N}};
\node (t1) at (1.7,0.02) {$\tensor$};
\node (n2) at (2.1,0) {\textsf{N}};
\node (t2) at (2.5,0.02) {$\tensor$};
\node (n3) at (2.9,0) {\textsf{N}};
\node (t3) at (3.3,0.02) {$\tensor$};
\node (n4) at (3.7,0) {\textsf{S}};
\node (n5) at (5,0) {\textsf{N}};
\node (n6) at (6.4,0) {\textsf{N}};
\node (t4) at (6.8,0.02) {$\tensor$};
\node (n7) at (7.2,0) {\textsf{N}};
\node (t5) at (7.6,0.02) {$\tensor$};
\node (n8) at (8,0) {\textsf{S}};

\node (n9) at (2.1,2) {};
\node (n10) at (2.1,-1.5) {};

\draw (n0) edge [bend left=90] node[above=2pt] {$i$} (n1);
\draw (n3) edge [bend left=90,pos=.2] node[above=2pt] {$k$} (n6);
\draw (n4) edge [bend left=90, pos=.7] node[above=2pt] {$l$} (n8);
\draw (n5) edge [bend left=90, pos=.8] node[above=2pt] {$m$} (n7);
\draw (n0) edge [bend right=90]  node[below=2pt] {$i$} (n1);
\draw (n3) edge [bend right=100, pos=.2] node[below=2pt] {$k$} (n7);
\draw (n4) edge [bend right=90,pos=.7] node[below=2pt] {$l$} (n8);
\draw (n5) edge [bend right=90] node[below=2pt] {$m$} (n6);

\draw (n2) edge node[right] {$j$} (n9);
\draw (n2) edge node[right] {$j$} (n10);

\end{tikzpicture}
\caption{Matching diagrams for Dutch derivational ambiguity. Object relative (top), 
\textsf{mannen$_{i}$ die$_{ijkl}$ vrouwen$_{m}$ haten$_{kml}$} versus
subject relative (bottom) \textsf{mannen$_{i}$ die$_{ijkl}$ vrouwen$_{m}$ haten$_{mkl}$}.}
\label{matching}
\end{center}
\end{figure}

\paragraph*{Open class items vs function words} For open class lexical items,
concrete meanings are obtained distributionally. For function words, the relative
pronoun in this case, it makes more sense to assign them an interpretation
independent of distributions. To capture the intersective interpretation of
restrictive relative clauses, Sadrzadeh et al \cite{sadrzadeh2013frobenius} propose
to interpret the relative pronoun with a map that extracts a vector in the noun space
from the relative clause body, and then combines this by elementwise multiplication
with the vector for the head noun. Their account depends on the identification
$\F{np}=\F{n}=\textsf{N}$: noun phrases and simple common nouns are interpreted in the same
space; it expresses the desired meaning recipe for the relative pronoun with the aid of
(some of) the Frobenius operations that are available in a compact closed category:

\begin{equation}\label{frobenius}
\begin{tabular}{cccc}
	$\Delta : A \pijl A \tensor A$ & $\mu : A \tensor A \pijl A$ & $\iota: A \pijl I$ &  $\zeta : I \pijl A$ \\
\end{tabular}
\end{equation}
In the case of $\textbf{FVect}$, $\Delta$ takes a vector and places its values on the diagonal of a square matrix, whereas $\mu$ extracts the diagonal from a square matrix.
The $\iota$ and $\zeta$ maps respectively sum the coefficients of a vector or introduce a vector with the value $1$ for all of its coefficients.

\begin{center}
	\begin{tabular}{l@{\hskip 1em}c@{\hskip 4em}ccl}
			$ \Delta_V : V \pijl V \tensor V$ & given by & $\sum\limits_i v_i \vec{e}_i$ & $\mapsto$ & $\sum\limits_i v_i (\vec{e}_i \tensor \vec{e}_i)$ \\[1em]
			$ \iota_V : V \pijl \reals$ & given by & $\sum\limits_i v_i \vec{e}_i$ & $\mapsto$ & $\sum\limits_i v_i $\\[1em]
			$ \mu_V : V \tensor V \pijl V$ & given by & $\sum\limits_{ij} v_{ij} (\vec{e}_i \tensor \vec{e}_j)$ & $\mapsto$ & $\sum\limits_{i} v_{ii} \vec{e}_i$ \\[1em]
			$ \zeta_V : \reals \pijl V$ & given by & $\lambda$ & $\mapsto$ & $\sum\limits_i \lambda \vec{e}_i$
	\end{tabular}
\end{center}

The analysis of \cite{sadrzadeh2013frobenius} uses a pregroup syntax and addresses relative clauses in English.
It relies on distinct pronoun types for subject and object relativisation. In the subject relativisation case,
the pronoun lives in the space $\textsf{N}\tensor\textsf{N}\tensor\textsf{S}\tensor\textsf{N}$, corresponding to
$n^{r}\,n\,s^{l}\,np$,
the pregroup translation of a Lambek type $(n\bs n)/(np\bs s)$; for object relativisation, the pronoun lives in
$\textsf{N}\tensor\textsf{N}\tensor\textsf{N}\tensor\textsf{S}$, corresponding to $n^{r}\,n\,np^{ll}\, s^{l}$,
the pregroup translation of $(n\bs n)/(s/np)$.

For the case of Dutch, the homomorphism $\F{\cdot}$ of \S\ref{subsect:derivational} sends the relative pronoun type
$(n\bs n)/(\fdia\gbox np\bs s)$ to the space $\textsf{N}\tensor\textsf{N}\tensor\textsf{N}\tensor\textsf{S}$.
This means we can import the 
pronoun interpretation for that space from \cite{sadrzadeh2013frobenius}, which now will produce
both the subject and object relativisation interpretations through its interaction with the
derivational semantics.

\begin{equation}\label{diemu}
\textbf{die} \quad=\quad (1_{\textsf{N}} \tensor \mu_{\textsf{N}}\tensor 1_{\textsf{N}} \tensor \zeta_{\textsf{S}}) \comp (\eta_{\textsf{N}} \tensor \eta_{\textsf{N}})
\end{equation}

Intuitively, the recipe (\ref{diemu}) says that the pronoun consists of a cube (in $\textsf{N}\tensor\textsf{N}\tensor\textsf{N}$)
which has $1$ on its diagonal and $0$ elsewhere, together with a vector in the sentence space $\textsf{S}$ with all its entries $1$. 
Substituting this lexical recipe in the tensor contraction equations of (\ref{subjectreading}) and (\ref{objectreading})
yields the desired final semantic values (\ref{subjectmeaning}) and (\ref{objectmeaning}) for subject and object relativisation respectively.
We write $\odot$ for elementwise multiplication; the summation over the \textsf{S} dimension reduces the rank-3 $\textsf{N}\tensor\textsf{N}\tensor\textsf{S}$
interpretation of the verb to a rank-2 matrix in $\textsf{N}\tensor\textsf{N}$, with rows for the verb's object, columns for the subject.
This matrix is applied to the vector $\mathbf{vrouwen}$
either forward in (\ref{objectmeaning}), where `vrouwen' plays the subject role, 
or backward in (\ref{subjectmeaning}) before being elementwise multiplied with the vector for \textsf{mannen}.
\begin{equation}\label{subjectmeaning}
	(\ref{subjectreading}) \quad=\quad \mathbf{mannen} \odot \Big[\Big(\sum\limits_{S} \mathbf{haten}\Big)^{\textsf{T}}\mathbf{vrouwen}\Big]
\end{equation}
\begin{equation}\label{objectmeaning}
	(\ref{objectreading}) \quad=\quad \mathbf{mannen} \odot \Big[\Big(\sum\limits_{S} \mathbf{haten}\Big)\mathbf{vrouwen}\Big]
\end{equation}

Returning to English, notice that the pregroup type assignment $n^{r}\,n\,np^{ll}\, s^{l}$ for object relativisation
in \cite{sadrzadeh2013frobenius}
is restricted to cases where the `gap' in the relative clause body occupies the final position. To cover these non-subject
relativisation patterns in general, also with respect to positions internal to the relative clause body, we would
use an $\textbf{NL}_{\diamond}$ type $(n\bs n)/(s/\fdia\gbox np)$ for the pronoun, together with the rightward extraction
postulates $\xrighta,\xrightc$ of Figure \ref{nldia}. For English subject relativisation, the simple pronoun type $(n\bs n)/(np\bs s)$
will do, as this pattern doesn't require any structural reasoning.

\section{Discussion}
\label{sect:discussion}
We briefly compare the distributional and the formal semantics accounts, highlighting their similarities.
In the formal semantics account, the interpretation homomorphism sends syntactic types to their semantic counterparts.
Syntactic types are built from atoms, for example $s$, $np$, $n$ for sentences, noun phrases and common nouns; assuming
semantic atoms $e$, $t$ and function types built from them, one can set $\F{s}=t$, $\F{np}=e$, $\F{n}=e\rightarrow t$,
and $\F{A/B}=\F{B\bs A}=\F{B}\rightarrow\F{A}$. Each semantic type $A$ is assigned an interpretation domain $D_A$,
with $D_{e} = E$, for some non-empty set $E$ (the discussion domain), $D_{t}=\{0,1\}$ (truth values), and $D_{A\rightarrow B}$
funtions from $D_A$ to $D_B$.

In this setup, a syntactic derivation $A_1,\ldots A_n\Rightarrow B$ is interpreted by means of a linear lambda
term $M$ of type $\F{B}$, with parameters $x_i$ of type $\F{A_i}$ --- linearity resulting from the fact that
the syntactic source doesn't provide the copying/deletion operations associated with the structural rules
of Contraction and Weakening.

As in the distributional model discussed here, the proof term $M$ is an instruction for meaning assembly
that abstracts from lexical semantics. In (\ref{proofterms}) below, one finds the proof terms for English
subject (a) and object (b) relativisation. The parameter $w$ stands for the head noun, $f$ for the verb,
$y$ and $z$ for its object and subject arguments; parameter $x$ for the relative pronoun has type
$(e\rightarrow t)\rightarrow (e\rightarrow t)\rightarrow e\rightarrow t$.

\begin{equation}\label{proofterms}
\begin{array}{cl@{\qquad}l}
(a) & n,(n\bs n)/(np\bs s), (np\bs s)/np, np\Rightarrow n & 
(x_{\mathit{who}}^{}\ \lambda z^{e}.(f_{\mathit{}}^{e\rightarrow e\rightarrow t}\ y^{e}\ z_{\mathit{}}^{e})\ w_{\mathit{}}^{e\rightarrow t})\\[2ex]
(b) & n,(n\bs n)/(s/np), np, (np\bs s)/np\Rightarrow n &
(x_{\mathit{who}}^{}\ \lambda y^{e}.(f_{\mathit{}}^{e\rightarrow e\rightarrow t}\ y^{e}\ z_{\mathit{}}^{e})\ w_{\mathit{}}^{e\rightarrow t})\\
\end{array}
\end{equation}

To obtain the interpretation of `men who hate women' vs `men who(m) women hate', one substitutes lexical meanings
for the parameters of the proof terms. In the case of the open class items `men', `hate', `women', these will
be non-logical constants with an interpretation depending on the model. For the relative pronoun, we substitute
an interpretation independent of the model, expressed in terms of the logical constant $\wedge$, leading to
the final interpretations of (\ref{finalterms}), after normalisation.

\begin{equation}\label{whoterm}
x_{\mathit{who}} := \lambda x^{e\rightarrow t}\lambda y^{e\rightarrow t}\lambda z^{e}.((x\ z) \wedge ((y\ z))
\end{equation}

\begin{equation}\label{finalterms}
\begin{array}{cl}
(a) & \lambda x.((\textsc{men}\ x) \wedge (\textsc{hate}\ \textsc{women}\ x))\\[2ex]
(b) & \lambda x.((\textsc{men}\ x) \wedge (\textsc{hate}\ x\ \textsc{women}))\\
\end{array}
\end{equation}
Notice that the lexical meaning recipe for the relative pronoun goes beyond linearity:
to express the set intersection interpretation, the bound $z$ variable is copied over the conjuncts of $\wedge$. 
By encapsulating this copying operation in the lexical semantics, one avoids compromising the derivational
semantics. In this respect, the formal semantics account makes the same design choice regarding the
division of labour between derivational and lexical semantics as the distributional account, where the
extra expressivity of the Frobenius operations is called upon for specifying the lexical meaning recipe
for the relative pronoun.

\section{Acknowledgments}
We thank Giuseppe Greco for comments on an earlier version.
The second author would also like to thank Mehrnoosh Sadrzadeh for the many discussions on compositional distributional semantics and Frobenius operations,
and Rob Klabbers for his interesting remarks on index notation. The second author gratefully acknowledges support by a Queen Mary Principal's Research Studentship,
the first author the support of the Netherlands Organisation for Scientific Research (NWO, Project 360-89-070,
\emph{A composition calculus for vector-based semantic modelling with a localization for Dutch}).
\label{sect:acks}

\label{sect:bib}
\bibliographystyle{plain}

\newpage
\appendix
\section{Simplifying the Interpretation}\label{simplappendix}
The simplification of section \ref{simpl} uses generalised Kronecker deltas to interpret the proof terms of the proof system, leading to a relabelling of the proof system with formulas interpreted. The rules that change the generalised Kronecker delta are shown in section \ref{simpl}, the full system is shown in Figure \ref{simplsystem}. In this appendix we show that the simplification holds.

\begin{figure}[h]
\[\infer[1_A]{\arrover{\delta^{I}_{J}}{A_I}{A_J}}{}\]

\[\infer[\resdia]{\arrover{\delta^I_J}{A}{ B}}{\arrover{\delta^I_J}{ A}{B}}
\qquad
\infer[\resright]{\arrover{\delta^I_J}{A}{C\tensor B}}{\arrover{\delta^I_J}{A\tensor B}{C}}
\qquad
\infer[\resleft]{\arrover{\delta^I_J}{B}{A\tensor C}}{\arrover{\delta^I_J}{A\tensor B}{C}}\]

\[\infer[\resdiainv]{\arrover{\delta^I_J}{ A}{B}}{\arrover{\delta^I_J}{A}{ B}}
\qquad
\infer[\resrightinv]{\arrover{\delta^I_J}{A\tensor B}{C}}{\arrover{\delta^I_J}{A}{C\tensor B}}
\qquad
\infer[\resleftinv]{\arrover{\delta^I_J}{A\tensor B}{C}}{\arrover{\delta^I_J}{B}{A\tensor C}}\]

\[\infer[\fdia]{\arrover{\delta^I_J}{ A}{ B}}{\arrover{\delta^I_J}{A}{B}}
\qquad
\infer[\gbox]{\arrover{\delta^I_J}{ A}{ B}}{\arrover{\delta^I_J}{A}{B}}\]

\[\infer[\tensor]{\arrover{\delta^{IK}_{JL}}{A\tensor C}{B\tensor D}}{\arrover{\delta^I_J}{A}{B} & \arrover{\delta^K_L}{C}{D}}
\qquad
\infer[\s]{\arrover{\delta^{IK}_{JL}}{A\tensor D}{B\tensor C}}{\arrover{\delta^I_J}{A}{B} & \arrover{\delta^K_L}{C}{D}}
\qquad
\infer[\bs]{\arrover{\delta^{IK}_{JL}}{B\tensor C}{A\tensor D}}{\arrover{\delta^I_J}{A}{B} & \arrover{\delta^K_L}{C}{D}}
\]

\[\infer[\Xlefta]{\arrover{\delta^I_J}{ A\tensor (B\tensor C)}{D}}{\arrover{\delta^I_J}{( A\tensor B)\tensor C}{D}}
\qquad
\infer[\Xrighta]{\arrover{\delta^I_J}{(A\tensor B)\tensor  C}{D}}{\arrover{\delta^I_J}{A\tensor(B\tensor  C)}{D}}
\]

\[\infer[\Xleftc]{\arrover{\delta^I_J}{ A\tensor (B\tensor C)}{D}}{\arrover{\delta^I_J}{B\tensor ( A\tensor C)}{D}}
\qquad
\infer[\Xrightc]{\arrover{\delta^I_J}{(A\tensor B)\tensor  C}{D}}{\arrover{\delta^I_J}{(A\tensor  C)\tensor B}{D}}
\]

\caption{\textbf{NL}$_{\diamond}$. Rules annotated with their generalised Kronecker deltas.}
\label{simplsystem}
\end{figure}

Analogous to $\F{f}$, we write $\G{f}$ for the generalised Kronecker delta associated with proof term $f$. We define the expressions of a compact closed category for generalised Kronecker deltas.
Then we show that for any proof $\arr{f}{A}{B}$ we have that $\F{f} =  \G{f}$. Proving this is done by induction over the size of proofs. The crucial point is that composition of two generalised Kronecker deltas is determined by their domain and codomain.

\paragraph*{The CCC structure of generalised Kronecker deltas} To give a generalised Kronecker delta an interpretation as a map, we need to give its domain and codomain. We will write for a generalised Kronecker delta, $\delta^I_J : A_M \pijl B_N$ to indicate that $A$ is the domain, $B$ the codomain, and moreover that concrete tensors in $A$ will be labelled with list $M$, and output tensors will be labelled with $N$. Writing $+$ for list concatenation and $\pi(L)$ for any permutation of list $L$, we then assume that $I + J = \pi(M+N)$. We can now go on and define the maps of a compact closed category in the generalised Kronecker delta form.

Note the way generalised Kronecker deltas are rewritten: a generalised Kronecker delta $\delta^I_J$ has pairs of indices on top and bottom that are linked. Whenever $\delta^I_J$ has an index occuring twice, a rewrite is done: let $(a,b),(c,d)$ be two pairs of indices, $a,c$ on top and $b,d$ on the bottom, that have an index in common, say $a = c$. Then we remove from $\delta^I_J$ the pair $(a,b)$, and replace the pair $(c,d)$ by $(b,d)$. This lowers the rank of the generalised Kronecker delta with 2, which is in line with the idea of the tensor contraction that is to be performed by the common index. This generalises to lists of indices: writing $(A,B)$ and $(C,D)$ for lists of indices such that pairs $(a,b)$ come from $A,B$ and pairs $(c,d)$ come from $C, D$, if we have that $A = C$ we can immediately remove $(A,B)$ and replace $(C,D)$ by $(B,D)$. The whole rewriting continues until there are only unique indices left. Write $(\delta^I_J)^{\ast}$ (or $\delta^{I^{\ast}}_{J^{\ast}}$) for the generalised Kronecker delta obtained by rewrites from the original $\delta^I_J$.
Then, for two generalised Kronecker deltas, we get the relation
	$$ \delta^I_J \delta^K_L = \delta^{(I+K)^{\ast}}_{(J+L)^{\ast}}$$
In particular this means for $\delta^I_J$ and $\delta^K_L$ with no indices in common that $\delta^{(I+K)^{\ast}}_{(J+L)^{\ast}} = \delta^{I+K}_{J+L}$. 

Another special case is when we have $\delta^{A+I}_{B+J}$ where $A$ occurs in $I + J$, but $B$ does not. Then we have that we remove $A$ and $B$ and, for each index $a$ in $I$ and $J$, we substitute the corresponding index $b$ in $B$:
	$$\delta^{(A+I)^{\ast}}_{(B+J)^{\ast}} = (\delta^I_J \{ A \mapsto B \})^{\ast}$$
When $I$ and $J$ have no elements in common, then the right hand side is already fully rewritten since $B$ is unique to $I$ and $J$, allowing us to drop the asterisk.
We use these properties in the below definition for tensor product and composition of generalised Kronecker deltas, and in the proof in the next paragraph.
\begin{definition} The maps in \textbf{FVect} are defined in terms of generalised Kronecker deltas according the the list below:
\begin{enumerate}
	\item For any vector space $V$ of rank $n$, the identity map $1_V : V \pijl V$ is given by the generalised Kronecker delta $$\delta^{i_1,i_2,...,i_n}_{j_1,j_2,...,j_n} : V_{i_1,i_2,...,i_n} \pijl V_{j_1,j_2,...,j_n}$$ On an element $\mathbf{v} \in V$, represented in index notation by $v_{i_1,...,i_n}$, we get simply a renaming because $$\delta^{i_1,i_2,...,i_n}_{j_1,j_2,...,j_n}v_{i_1,...,i_n} = v_{j_1,...,j_n}$$
	\item For any vector space $V$ of rank $n$, the $\epsilon_V : V \tensor V \pijl \reals$ map is given by the generalised Kronecker delta $$\delta^{i_1,i_2,...,i_n}_{j_1,j_2,...,j_n} : V_{i_1,i_2,...,i_n} \tensor V_{j_1,j_2,...,j_n} \pijl \reals$$ For two elements $\mathbf{v} \in V$ and $\mathbf{w} \in V$ represented by $v_{i_1,...,i_n}$ and $w_{j_1,...,j_n}$ we get the inner product between $\mathbf{v}$ and $\mathbf{w}$: $$\delta^{i_1,i_2,...,i_n}_{j_1,j_2,...,j_n}v_{i_1,...,i_n}w_{j_1,...,j_n} = v_{j_1,...,j_n}w_{j_1,...,j_n}$$
	\item For any vector space $V$ of rank $n$, the $\eta_V : \reals \pijl V \tensor V$ map is given by $$\delta^{i_1,i_2,...,i_n}_{j_1,j_2,...,j_n} : \reals \pijl V_{i_1,i_2,...,i_n} \tensor V_{j_1,j_2,...,j_n}$$ It is given no elements to juxtapose with and thus simply gives the identity matrix on $V$.
	\item Composition. Given two maps (left) and their generalised Kronecker delta representation (right)
		$$f : A \pijl B \qquad \delta^I_J : A_M \pijl B_{N_1}$$
		$$g : B \pijl C \qquad \delta^K_L : B_{N_2} \pijl C_O$$
		their composition $g \comp f : A \pijl C$ is represented by $$\delta^{(N_1 + I + K)^{\ast}}_{(N_2 + J + L)^{\ast}} : A_M \pijl C_O$$ We give the expression $\delta^{N_1}_{N_2} \delta^I_J \delta^K_L$ for the composition, exactly to identify the indices in the codomain of $f$ ($N_1$) with the indices in the domain of $g$ ($N_2$). Since $\delta^I_J$ and $\delta^K_L$ have no indices in common (this may be assumed without loss of generality), we have $\delta^I_J \delta^K_L = \delta^{I+K}_{J+L}$, but since $N_2$ occurs in $K+L$ and $N_1$ occurs in $I+J$, we will have a sequence of rewrites to do and so we get $\delta^{(N_1 + I + K)^{\ast}}_{(N_2 + J + L)^{\ast}}$.
	\item Tensor product. Given two maps (left) and their generalised Kronecker delta representation (right)
		$$f : A \pijl B \qquad \delta^I_J : A_M \pijl B_N$$
		$$g : C \pijl D \qquad \delta^K_L : C_O \pijl D_P$$
		their tensor product $f \tensor g : A \tensor C \pijl B \tensor D$ is represented by $$\delta^{I+J}_{K+L} : A_M \tensor C_O \pijl B_N \tensor D_P$$ Without loss of generality we may assume that $\delta^I_J$ and $\delta^K_L$ have no indices in common (if they had, we could rename them). Since $I + J = \pi(M+N)$ and $K + L = \pi(O + P)$, we also have that $I + J + K + L = \pi(M + N + O + P)$. And since $\delta^I_J,\delta^K_L$ have no indices in common, we have that juxtaposing them gives $\delta^I_J\delta^K_L = \delta^{I+J}_{K+L}$.
	\item Associativity. Since the tensor product is associative on vectors, the associativity maps disappear in index notation. For vector spaces $A,B,C$ of rank $k,l,m$ respectively, the associativity map $\alpha_{A,B,C} : (A \tensor B) \tensor C \pijl A \tensor (B \tensor C)$ is represented as
		$$\delta^{M_1N_1O_1}_{M_2N_2O_2} : A_{M_1} \tensor B_{N_1} \tensor C_{O_1} \pijl A_{M_2} \tensor B_{N_2} \tensor C_{O_2}$$
		where $M_1,M_2$ have length $k$, $N_1,N_2$ have length $l$ and $O_1,O_2$ have length $m$.
		This acts simply as an identity map: on elements $\mathbf{a} \in A,\mathbf{b} \in B, \mathbf{c} \in C$, we get
		$$\delta^{M_1N_1O_1}_{M_2N_2O_2} a_{M_1} b_{N_1} c_{O_1} = a_{M_2} b_{N_2} c_{O_2}$$
		which is simply a renaming of the input. The inverse associativity map $\alpha^{-1} : A \tensor (B \tensor C) \pijl (A \tensor B) \tensor C$ is represented exactly the same and again works as an identity map.
	\item Symmetry. The tensor product is not commutative on vectors, but the generalised Kronecker delta for the symmetry map $\sigma_{A,B} : A \tensor B \pijl B \tensor A$ on vector spaces $A,B$ with rank $k,l$ respectively, performs an identity. The order of evaluation is given by the switch in indices in input and output. So $\sigma$ is represented by $$\delta^{M_1N_1}_{M_2N_2} : A_{M_1} \tensor B_{N_1} \pijl B_{N_2} \tensor A_{M_2}$$ On an input $\mathbf{a} \in A, \mathbf{b} \in B$, we get
		$$\delta^{M_1N_1}_{M_2N_2} a_{M_1}b_{N_1} = a_{M_2}b_{N_2} $$
		but here the order of the indices in the codomain dictates that the elements of $\mathbf{a}$ are placed after the elements of $\mathbf{b}$.
\end{enumerate}
\end{definition}

\paragraph*{Reducing the interpretation}
With the translation of maps of \textbf{FVect} in terms of generalised Kronecker deltas, we are ready to state our claim.

\begin{theorem} For any proof $\arr{f}{A}{B}$, we have that $\F{f} = \G{f}$.
\end{theorem}
\begin{proof} By induction over the size of proofs. The base case is the case of the axiom, for which we have $\F{1_A} = 1_{\F{A}}$. The identity map is represented by $\delta^I_J : \F{A}_I \pijl \F{A}_J$ in the definition above. For the inductive step, we proceed by cases. Just as in the main text, for brevity we will suppress the $+$ symbol in list concatenation. We moreover write the composition symbol $\comp$ in between generalised Kronecker deltas when its is clear what is intended.
	\begin{enumerate}
		\item Residuation. There are six subcases to consider:
			\begin{enumerate}
				\item $\resleft f$: assuming that $\F{f} : \F{A \tensor B} \pijl \F{C}$ is equal to $\delta^I_J : \F{A}_M \tensor \F{B}_N \pijl \F{C}_O$, we need to show that $\F{\resleft f} : \F{B} \pijl \F{A \bs C}$ is equal to $\delta^I_J : \F{B}_N \pijl \F{A}_M \tensor \F{C}_O$. We compute
					$$\F{\resleft f} = (1_{\F{A}}\otimes\F{f}) \comp (\eta_{\F{A}}\otimes 1_{\F{B}})$$
					$$= (\delta^{M_4}_{M_5} \delta^I_J) \comp (\delta^{M_2}_{M_3} \delta^{N_2}_{N_3})$$
					$$= \delta^{M_4I}_{M_5J} \comp \delta^{M_2N_2}_{M_3N_3}$$
				which, by the codomain of $\eta_{\F{A}}\otimes 1_{\F{B}}$ and domain of $1_{\F{A}}\otimes\F{f}$ ($A \tensor A \tensor B$) means replacing as follows:
					$$= \delta^{M_2M_3N_3}_{M_4M\ N} \delta^{M_2N_2}_{M_3N_3}\delta^{M_4I}_{M_5J} $$
					$$= (\delta^{M_2M_3N_3M_2N_2M_4I}_{M_4M\ N\ M_3N_3M_5J})^{\ast} $$
					$$= (\delta^{M_3N_3M_4N_2M_4I}_{M\ N\ M_3N_3M_5J})^{\ast} $$
					$$= (\delta^{N_3M_4N_2M_4I}_{N\ M\ N_3M_5J})^{\ast} $$
					$$= (\delta^{M_4N_2M_4I}_{M\ N\ M_5J})^{\ast} $$
					$$= (\delta^{N_2M\ I}_{N\ M_5J})^{\ast} $$
					$$= \delta^{I}_{J} \{N \mapsto N_2, M \mapsto M_5\}$$
				So this gives a renaming of the original map, hence we need to rename the domain and codomain to $\F{B}_{N_2}, \F{A}_{M_5}, \F{C}_O$ but we can then also rename back to $\delta^I_J : \F{B}_N \pijl \F{A}_M \tensor \F{C}_O$.
				\item $\resleft^{-1} g$: assuming that $\F{g} : \F{B} \pijl \F{A \bs C}$ is equal to $\delta^I_J : \F{B}_N \pijl \F{A}_M \tensor \F{C}_O$ we need to show that $\F{\resleft^{-1} g} : \F{A \tensor B} \pijl \F{C}$ is equal to $\delta^I_J : \F{A}_M \tensor \F{B}_N \pijl \F{C}_O$. We compute
					$$\F{\resleft^{-1} g} = (\epsilon_{\F{A}}\tensor 1_{\F{C}}) \comp (1_{\F{A}}\tensor\F{g})$$
					$$= (\delta^{M_4}_{M_5} \delta^{O_2}_{O_3}) \comp (\delta^{M_2}_{M_3} \delta^{I}_{J})$$
					$$= \delta^{M_4O_2}_{M_5O_3} \comp \delta^{M_2I}_{M_3J}$$
				which, by the codomain of $1_{\F{A}}\tensor\F{g}$ and domain of $\epsilon_{\F{A}}\tensor 1_{\F{C}}$ ($A \tensor A \tensor C$) means replacing as follows:
					$$= \delta^{M_3M\ O}_{M_4M_5O_2} \delta^{M_2I}_{M_3J} \delta^{M_4O_2}_{M_5O_3}$$
					$$= (\delta^{M_3M\ O\ M_2IM_4O_2}_{M_4M_5O_2M_3JM_5O_3})^{\ast}$$
					$$= (\delta^{M\ O\ M_2IM_4O_2}_{M_5O_2M_4JM_5O_3})^{\ast}$$
					$$= (\delta^{O\ M_2IM_4O_2}_{O_2M_4JM\ O_3})^{\ast}$$
					$$= (\delta^{M_2IM_4O\ }_{M_4JM\ O_3})^{\ast}$$
					$$= (\delta^{IM_2O\ }_{JM\ O_3})^{\ast}$$									
					$$= \delta^I_J \{M \mapsto M_2, O \mapsto O_3\}$$
				which gives just a renaming of $\delta^I_J$.
				\item The cases of $\resright f, \resright^{-1} g$ are exactly symmetric to the cases of $\resleft, \resleft^{-1}$.
				\item $\resdia f, \resdiainv f$: since $\F{\resdia f} = \F{f} = \F{\resdiainv f}$, these cases are immediate. 
			\end{enumerate}
		\item Monotonicity. Here we need to consider five subcases.
			\begin{enumerate}
				\item $\Diamond f, \Box f$. These cases are again immediate since $\F{\Diamond f} = \F{f} = \F{\Box f}$.
				\item $f \tensor g$: since $\F{f \tensor g} = \F{f} \tensor \F{g}$ this case is immediate since the tensor product of generalised Kronecker deltas is performed by simply concatenating the upper and lower index lists together.
				\item $f \bs g$: assuming $\F{f} : \F{A} \pijl \F{B}$ and $\F{g} : \F{C} \pijl \F{D}$ are equal to $\delta^I_J : \F{A}_M \pijl \F{B}_N$ and $\delta^K_L : \F{C}_O \pijl \F{D}_P$ we need to show that $\F{f \bs g} : \F{B \bs C} \pijl \F{A \bs D}$ is equal to $\delta^{IK}_{JL} : \F{B}_N \tensor \F{C}_O \pijl \F{A}_M \tensor \F{D}_P$. We compute
					$$ \F{f \bs g} = (\epsilon_{\F{B}}\tensor 1_{\F{A}\tensor\F{D}}) \comp (1_{\F{B}}\tensor\F{f}\tensor 1_{\F{A}\tensor\F{D}}) \comp (1_{\F{B}}\tensor\eta_{\F{A}}\tensor \F{g})$$
					with (co)domain signature
					$$\F{B}\tensor\F{C} \pijl \F{B}\tensor\F{A}\tensor\F{A}\tensor\F{D} \pijl \F{B}\tensor\F{B}\tensor\F{A}\tensor\F{D} \pijl \F{A}\tensor\F{D}$$
					So writing the composition with generalised Kronecker deltas, we get
					$$= (\delta^{N_6}_{N_7}\delta^{M_6}_{M_7}\delta^{P_4}_{P_5}) \comp (\delta^{N_4}_{N_5}\delta^I_J\delta^{M_4}_{M_5}\delta^{P_2}_{P_3}) \comp (\delta^{N_2}_{N_3}\delta^{M_2}_{M_3}\delta^K_L)$$
					$$= \delta^{N_6M_6P_4}_{N_7M_7P_5} \comp \delta^{N_4IM_4P_2}_{N_5JM_5P_3} \comp \delta^{N_2M_2K}_{N_3M_3L}$$
					
					The composition is
					$$= \delta^{N_5N\ M_5P_3}_{N_6N_7M_6P_4} \delta^{N_3M_2M_3P\ }_{N_4M\ M_4P_2} \delta^{N_2M_2K}_{N_3M_3L} \delta^{N_4IM_4P_2}_{N_5JM_5P_3} \delta^{N_6M_6P_4}_{N_7M_7P_5}$$
					$$= (\delta^{N_5N\ M_5P_3N_3M_2M_3P\ N_2M_2KN_4IM_4P_2N_6M_6P_4}_{N_6N_7M_6P_4N_4M\ M_4P_2N_3M_3LN_5JM_5P_3N_7M_7P_5})^{\ast}$$
					$$= (\delta^{N\ M_5P_3N_3M_2M_3P\ N_2M_2KN_4IM_4P_2N_6M_6P_4}_{N_7M_6P_4N_4M\ M_4P_2N_3M_3LN_6JM_5P_3N_7M_7P_5})^{\ast}$$
					$$= (\delta^{M_5P_3N_3M_2M_3P\ N_2M_2KN_4IM_4P_2N_6M_6P_4}_{M_6P_4N_4M\ M_4P_2N_3M_3LN_6JM_5P_3N\ M_7P_5})^{\ast}$$
					$$= (\delta^{P_3N_3M_2M_3P\ N_2M_2KN_4IM_4P_2N_6M_6P_4}_{P_4N_4M\ M_4P_2N_3M_3LN_6JM_6P_3N\ M_7P_5})^{\ast}$$
					$$= (\delta^{N_3M_2M_3P\ N_2M_2KN_4IM_4P_2N_6M_6P_4}_{N_4M\ M_4P_2N_3M_3LN_6JM_6P_4N\ M_7P_5})^{\ast}$$
					$$= (\delta^{M_2M_3P\ N_2M_2KN_4IM_4P_2N_6M_6P_4}_{M\ M_4P_2N_4M_3LN_6JM_6P_4N\ M_7P_5})^{\ast}$$
					$$= (\delta^{M_3P\ N_2M\ KN_4IM_4P_2N_6M_6P_4}_{M_4P_2N_4M_3LN_6JM_6P_4N\ M_7P_5})^{\ast}$$
					$$= (\delta^{P\ N_2M\ KN_4IM_4P_2N_6M_6P_4}_{P_2N_4M_4LN_6JM_6P_4N\ M_7P_5})^{\ast}$$
					$$= (\delta^{N_2M\ KN_4IM_4P\ N_6M_6P_4}_{N_4M_4LN_6JM_6P_4N\ M_7P_5})^{\ast}$$
					$$= (\delta^{M\ KN_2IM_4P\ N_6M_6P_4}_{M_4LN_6JM_6P_4N\ M_7P_5})^{\ast}$$
					$$= (\delta^{KIM\ P\ N_2M_6P_4}_{LJM_6P_4N\ M_7P_5})^{\ast}$$
					$$= (\delta^{KIP\ N_2M\ P_4}_{LJP_4N\ M_7P_5})^{\ast}$$
					$$= (\delta^{KIP\ N_2M\ P_4}_{LJP_4N\ M_7P_5})^{\ast}$$
					$$= (\delta^{KIN_2M\ P}_{LJN\ M_7P_5})^{\ast}$$
					$$= \delta^{IK}_{JL} \{N \mapsto N_2, M \mapsto M_7, P \mapsto P_5\}$$
				\item $f \s g$: assuming $\F{f} : \F{A} \pijl \F{B}$ and $\F{g} : \F{C} \pijl \F{D}$ are equal to $\delta^I_J : \F{A}_M \pijl \F{B}_N$ and $\delta^K_L : \F{C}_O \pijl \F{D}_P$ we need to show that $\F{f \s g} : \F{A \s D} \pijl \F{B \s C}$ is equal to $\delta^{IK}_{JL} : \F{A}_M \tensor \F{D}_P \pijl \F{B}_N \tensor \F{C}_O$. We have the composition
					$$\F{f \s g} = (1_{\F{B}\tensor\F{C}}\tensor\epsilon_{\F{D}}) \comp (1_{\F{B}\tensor\F{C}}\tensor\F{g}\tensor 1_{\F{D}}) \comp (\F{f}\tensor\eta_{\F{C}}\tensor 1_{\F{D}})$$
					with (co)domain signature
					$$\F{A}\tensor\F{D} \pijl \F{B}\tensor\F{C}\tensor\F{C}\tensor\F{D} \pijl \F{B}\tensor\F{C} \tensor\F{D}\tensor\F{D} \pijl \F{B}\tensor\F{C}$$
					which gets interpreted as
					$$= (\delta^{N_4}_{N_5}\delta^{O_6}_{O_7}\delta^{P_6}_{P_7}) \comp (\delta^{N_2}_{N_3}\delta^{O_4}_{O_5}\delta^K_L\delta^{P_4}_{P_5}) \comp (\delta^I_J\delta^{O_2}_{O_3}\delta^{P_2}_{P_3})$$
					$$= \delta^{N_4O_6P_6}_{N_5O_7P_7} \comp \delta^{N_2O_4KP_4}_{N_3O_5LP_5} \comp \delta^{IO_2P_2}_{JO_3P_3}$$
					$$= \delta^{N_3O_5P\ P_5}_{N_4O_6P_6P_7} \delta^{N\ O_2O_3P_3}_{N_2O_4O\ P_4} \delta^{IO_2P_2}_{JO_3P_3} \delta^{N_2O_4KP_4}_{N_3O_5LP_5} \delta^{N_4O_6P_6}_{N_5O_7P_7}$$
					$$= (\delta^{N_3O_5P\ P_5N\ O_2O_3P_3IO_2P_2N_2O_4KP_4N_4O_6P_6}_{N_4O_6P_6P_7N_2O_4O\ P_4JO_3P_3N_3O_5LP_5N_5O_7P_7})^{\ast}$$
					$$= (\delta^{O_5P\ P_5N\ O_2O_3P_3IO_2P_2N_2O_4KP_4N_4O_6P_6}_{O_6P_6P_7N_2O_4O\ P_4JO_3P_3N_4O_5LP_5N_5O_7P_7})^{\ast}$$
					$$= (\delta^{P\ P_5N\ O_2O_3P_3IO_2P_2N_2O_4KP_4N_4O_6P_6}_{P_6P_7N_2O_4O\ P_4JO_3P_3N_4O_6LP_5N_5O_7P_7})^{\ast}$$
					$$= (\delta^{P_5N\ O_2O_3P_3IO_2P_2N_2O_4KP_4N_4O_6P}_{P_7N_2O_4O\ P_4JO_3P_3N_4O_6LP_5N_5O_7P_7})^{\ast}$$
					$$= (\delta^{N\ O_2O_3P_3IO_2P_2N_2O_4KP_4N_4O_6P}_{N_2O_4O\ P_4JO_3P_3N_4O_6LP_7N_5O_7P_7})^{\ast}$$
					$$= (\delta^{O_2O_3P_3IO_2P_2N\ O_4KP_4N_4O_6P}_{O_4O\ P_4JO_3P_3N_4O_6LP_7N_5O_7P_7})^{\ast}$$
					$$= (\delta^{O_3P_3IO_4P_2N\ O_4KP_4N_4O_6P}_{O\ P_4JO_3P_3N_4O_6LP_7N_5O_7P_7})^{\ast}$$
					$$= (\delta^{P_3IO_4P_2N\ O_4KP_4N_4O_6P}_{P_4JO\ P_3N_4O_6LP_7N_5O_7P_7})^{\ast}$$
					$$= (\delta^{IO_4P_2N\ O_4KP_4N_4O_6P}_{JO\ P_4N_4O_6LP_7N_5O_7P_7})^{\ast}$$
					$$= (\delta^{IP_2N\ O\ KP_4N_4O_6P}_{JP_4N_4O_6LP_7N_5O_7P_7})^{\ast}$$
					$$= (\delta^{IN\ O\ KP_2N_4O_6P}_{JN_4O_6LP_7N_5O_7P_7})^{\ast}$$
					$$= (\delta^{IO\ KP_2N\ O_6P}_{JO_6LP_7N_5O_7P_7})^{\ast}$$					
					$$= (\delta^{IKP_2N\ O\ P}_{JLP_7N_5O_7P_7})^{\ast}$$					
					$$= (\delta^{IKN\ O\ P}_{JLN_5O_7P_2})^{\ast}$$											
					$$= (\delta^{IKN\ O\ P}_{JLN_5O_7P_2})^{\ast}$$											
					$$= \delta^{IK}_{JL} \{ N \mapsto N_5, O \mapsto O_7, P \mapsto P_2\}$$
			\end{enumerate}
		\item Structural Rules. All four cases are immediate: for the $\Xlefta f$ and $\Xrighta f$ rules we have that associativity is already built into generalised Kronecker deltas so precomposing with an associativity map does not change the generalised Kronecker delta. For the slightly more complicated $\Xleftc f$ and $\Xrightc f$ rules we have that the associativity maps do not change anything, but moreover the symmetry map is effectuated by switching the order of the arguments to the generalised Kronecker delta (the domain), so the generalised Kronecker delta itself does not change.
	\end{enumerate}
\end{proof}

\end{document}